\documentclass{article}

\usepackage{microtype}
\usepackage{graphicx}
\usepackage{subfigure}
\usepackage{booktabs} 

\usepackage{hyperref}

\usepackage[accepted]{icml2025}

\usepackage{amsmath}
\usepackage{amssymb}
\usepackage{mathtools}
\usepackage{amsthm}

\usepackage[capitalize,noabbrev]{cleveref}

\theoremstyle{plain}
\newtheorem{theorem}{Theorem}[section]

\theoremstyle{definition}

\theoremstyle{remark}

\usepackage[textsize=tiny]{todonotes}

\icmltitlerunning{Submission and Formatting Instructions for ICML 2025}

\begin{document}

\twocolumn[
\icmltitle{Gradient Based Method for the Fusion of Lattice Quantizers}



\icmlsetsymbol{equal}{*}

\begin{icmlauthorlist}
\icmlauthor{Liyuan Zhang}{equal,yyy}
\icmlauthor{Hanzhong Cao}{equal,yyy}
\icmlauthor{Jiaheng Li}{equal,yyy}
\icmlauthor{Minyang Yu}{equal,yyy}
\end{icmlauthorlist}
\icmlaffiliation{yyy}{Department of Electronic Engineering and Computer Science, Peking University, China}
\icmlcorrespondingauthor{Liyuan Zhang}{zly2003@stu.pku.edu.cn}
\icmlcorrespondingauthor{Hanzhong Cao}{xiaxiaoguang@stu.pku.edu.cn}
\icmlcorrespondingauthor{Minyang Yu}
{flaricy@stu.pku.edu.cn}
\icmlcorrespondingauthor{Jiaheng Li}{lijiaheng2022@stu.pku.edu.cn}

\icmlkeywords{Machine Learning, ICML}

\vskip 0.3in
]



\printAffiliationsAndNotice{\icmlEqualContribution} 

\begin{abstract}

In practical applications, lattice quantizers leverage discrete lattice points to approximate arbitrary points in the lattice. An effective lattice quantizer significantly enhances both the accuracy and efficiency of these approximations. In the context of high-dimensional lattice quantization, previous work proposed utilizing low-dimensional optimal lattice quantizers and addressed the challenge of determining the optimal length ratio in orthogonal splicing. Notably, it was demonstrated that fixed length ratios and orthogonality yield suboptimal results when combining low-dimensional lattices. Building on this foundation, another approach employed gradient descent to identify optimal lattices, which inspired us to explore the use of neural networks to discover matrices that outperform those obtained from orthogonal splicing methods. We propose two novel approaches to tackle this problem: the Household Algorithm and the Matrix Exp Algorithm. Our results indicate that both the Household Algorithm and the Matrix Exp Algorithm achieve improvements in lattice quantizers across dimensions 13, 15, 17 to 19, 21, and 22. Moreover, the Matrix Exp Algorithm demonstrates superior efficacy in high-dimensional settings.

\end{abstract}

\section{Introduction}

A lattice is defined as a set of linearly independent vectors in $\mathbb{R}^n$. In an $n$-dimensional space, most points cannot be represented using finite decimal coordinates. However, we can approximate these coordinates by expressing them in terms of lattice points. Specifically, given a lattice formed by $n$ vectors denoted as $(a_1, a_2, \ldots, a_n)$, for any point $x$ in this space, we can find a set of integers $(z_1, z_2, \ldots, z_n)$ such that the expression$||x - \sum_{i=1}^{n} z_i a_i||_2$is minimized. The tuple $(z_1, z_2, \ldots, z_n)$ represents an approximate coordinate representation of the point $x$. Our objective is to select a lattice that minimizes the error associated with this approximation.

The optimal lattice quantizer is defined as the lattice that achieves the minimum mean square error (MSE). This is equivalent to minimizing the normalized second moment (NSM), which serves as a scale-invariant measure of the mean square error.

Lattices have widespread applications across various fields, including digital communications, experimental design, data analysis, and particle physics.

The structure of the paper is organized as follows: Section II provides the theoretical background and relevant matrix knowledge essential for solving the problem at hand. In Section III, we develop our new algorithm theoretically. Section IV and the appendix describe our experimental setup and present the results, demonstrating that our algorithm is both theoretically sound and practically effective. Section V offers explanation of the advantages and shortcoming of our method. Finally, Section VI concludes our research.

\subsection{Related Work}

From the perspective of theory, the paper \cite{best} defines the local optimal lattice quantizer by using a method similar to defining the local minimum value of a function. Local optimal lattice quantizer is a lattice quantizer that satisfies the requirement that NSM will not decrease after a lattice matrix is left multiplied by a matrix that is infinitely tending to the identity matrix. On this basis, \cite{best} proved that all Voronoi regions of local optimal lattice quantizer satisfy some symmetry, that is, the correlation matrix is a constant multiple of the unit matrix. This proves theoretically that the Voronoi region of the optimal lattice quantizer must have a certain degree of symmetry.

Paper \cite{optimization} considers using lower triangular matrix to represent lattice quantizer matrix, using stochastic gradient descent algorithm to optimize lattice NSM, and proposes a powerful tool for converting numeric lattice representations into their underlying exact forms.

\cite{best} considers splitting the entire $n$-dimensional space into several subspace when designing an $n$-dimensional lattice. The optimal results of these subspace are then orthogonal concatenated. After realizing that orthogonality is the worst allocation method, we decided to use gradient descent to explore non-orthogonal cases.

Specifically, we referred to the optimal method of constructing an $n$-dimensional lattice from two low-dimensional lattices, as described in \cite{best}. According to the formulas in \cite{best}, given $k$ lattices, denoted as $A_i$ with volume $V_i$ and normalized $n$-sphere measure (NSM) as $G_i$, the best orthogonal concatenation $a_1A_1 \otimes a_2A_2 \otimes \dots \otimes a_kA_k$ must satisfy:

\[
a_i = \frac{C}{\sqrt{G_i}V_i^{\frac{1}{n}}}
\]

where $C$ is a constant.

Since $K_{12}$ is used in dimensions $13, 14, 15$, and $\Lambda_{16}$ is used in dimensions $17$ to $22$, we fixed the coefficients of these two matrices to $1$, manually computed the coefficients $a_i$ for the other matrices, and obtained the best matrices under orthogonal concatenation.

\subsection{Innovation of our work}

We propose a \textbf{gradient fusion method }for low-dimensional lattice quantizers, leveraging the properties of orthogonal transformations to achieve optimal performance.

Our experiments with Householder reflection matrices, which \textbf{maintain orthogonality throughout training}, achieved the best results, demonstrating the effectiveness of our approach.

Additionally, we conducted general experiments using \textbf{an initially orthogonal matrix}, focusing on the principle of combining low-dimensional matrices. These results further validate the effectiveness of our method.

\section{Theoretical Preparation}
\label{1}

\subsection{Defination}

\subsection{HouseHold Transform}
HouseHold Transform is a common way to generate orthogonal matrices.The reflection hyperplane can be defined by its \textit{normal vector}, a \textit{unit vector} \(\mathbf{v}\) (a vector with length 1) that is orthogonal to the hyperplane. The reflection of a point \(\mathbf{x}\) about this hyperplane is the linear transformation:

\[
\mathbf{x} - 2\langle \mathbf{x}, \mathbf{v} \rangle \mathbf{v} = \mathbf{x} - 2\mathbf{v} (\mathbf{v}^* \mathbf{x}),
\]

where \(\mathbf{v}\) is given as a column unit vector with \textit{conjugate transpose} \(\mathbf{v}^*\).

The matrix constructed from this transformation can be expressed in terms of an \textit{outer product} as:

\[
P = I - 2\mathbf{v}\mathbf{v}^*,
\]

is known as the \textit{Householder matrix}, where \(I\) is the \textit{identity matrix}.

As for its ability to generate orthogonal matrices, we have the following matrice:
\begin{theorem}
    Any orthogonal matrix of size n × n can be constructed as a product of at most n such reflections.
\end{theorem}
\begin{proof}
You can find the proof in all kinds of algebra books like \cite{matrix}
\end{proof}

\subsection{Matrix Exponential}

Another method to generate orthogonal matrices is exponential transformation. The theory behind is the relationship between orthogonal transformation, Lie Group and Lie Algebra. To be brief, we have the following theorems.

\begin{theorem}
    If $A$ is a real anti-symmetric matrix, then $\exp(A)$ is a real orthogonal matrix and $\det\left(\exp(A)\right)=1$.
\end{theorem}

\begin{proof}
    Let $A$ be a real anti-symmetric matrix, {\it i.e.} $A = -A^T$. Let $Q = \exp(A)$. Then 
    $Q^T Q= \exp(A)^T \exp(A)=\exp(A^T)\exp(A)= \exp\bigl(A^T + A\bigr)= \exp(0_{n\times n})= I$. This indicates that $Q$ is orthogonal.

    The following proof uses the property that $\exp(A)^T = \exp(A^T)$.

    For any matrix $M$, we have
\[
\det\bigl(\exp(M)\bigr) \;=\; \exp\bigl(\mathrm{tr}(M)\bigr).
\]

Thus, $\det(Q) = \exp(\mathrm{tr}(A)) = \exp(0)=1$
\end{proof}

\begin{theorem}
    For any real matrix $T\in SO(n)$, {\it i.e.} $T\cdot T^T = 1$ and $det(T) = 1$, we can find real anti-symmetric matrix $A$ such that $\exp(A) = T$.
\end{theorem}
\begin{proof}
    For any $T \in SO(n)$, all its eigenvalues in $\mathbb{C}$ have length $1$. Since $T$ is real, all its complex eigenvalues are conjugate. Thus, $T$ can be orthogonally decomposed as
    \[
    T
    \;\sim\;
    \mathrm{diag}\bigl(R(\theta_1), R(\theta_2), \dots, R(\theta_k), 1, \dots, 1\bigr),
    \] 
    where 
$R(\theta)
\;=\;
\begin{pmatrix}
\cos \theta & -\sin \theta \\
\sin \theta & \cos \theta
\end{pmatrix}.$

    The diagonal elements in real parts are $1$ because $det(T) = 1$. This decomposition means there exists $P\in O(n)$ such that $T = P\cdot \mathrm{diag}\bigl(R(\theta_1), R(\theta_2), \dots, R(\theta_k), 1, \dots, 1\bigr)\cdot P^T$. 

    Let $A_j  
\;=\;
\begin{pmatrix}
0 & -\theta_j\\
\theta_j & 0
\end{pmatrix}
$ for $j = 1,2,...,k$. And take $A$ to be 
\[
A
\;=\;
\mathrm{diag}\bigl(A_1, A_2, \dots, A_k, 0, \dots, 0\bigr),
\]

Then
\[
\exp(A)
\;=\;
\mathrm{diag}\bigl(R(\theta_1), \dots, R(\theta_k), I, \dots, I\bigr).
\]
Finally, using the property that $\exp(P AP^{-1})=P\exp(A)P^{-1}$, we take $A' = PAP^T$ to get 
\[
\exp(A')
\;=\; \exp\bigl(PAP^T\bigr)
\;=\; P\exp(A)P^T
\;=\; T.
\]
\end{proof}

These two theorems show that we can use real anti-symmetric matrices to generate any orthogonal transformation with determinant $1$, by applying the exponential operation.

\section{Method}

In previous studies, the Cartesian product of two lattices (or other sets of vectors) has been widely used in generative tasks for lattices. The Cartesian product of two lattices is defined as follows:

\[
L_{1} \times L_{2} \triangleq \{ [{\boldsymbol{x}}_{1} \; {\boldsymbol{x}}_{2}] \colon {\boldsymbol{x}}_{1} \in L_{1}, {\boldsymbol{x}}_{2} \in L_{2} \}. \tag{4}
\]

As observed in \cite{best}, the best-performing lattices in lower dimensions are often used to generate lattices in higher dimensions via the Cartesian product. However, this generative approach is not always optimal, as the NSM (normalized second moment) of these lattices can often be reduced by applying small rotations to their corresponding generator matrices, as indicated in \cite{best}. To address this limitation, \cite{glued} introduces the concept of glued vectors to enforce non-orthogonal relationships (as will be illustrated later with an example, which shows that this approach is equivalent to applying a rotation to the generator matrix). Using this method, \cite{glued} successfully achieved state-of-the-art results for lattices in 12 dimensions.

Recently, machine learning techniques have been increasingly applied to lattice generation and optimization. In \cite{optimization}, a stochastic gradient descent (SGD) approach is used to iteratively improve the generator matrix by computing the gradient of NSM with respect to the generator matrix's parameters. This approach has proven effective in approximating the optimal solution and has achieved state-of-the-art results in 15 dimensions.

\subsection{Learnable Symmetry Matrix}

Inspired by the above two approaches, we first propose a novel optimization method. This method involves applying rotations to the lattice and improving the corresponding generator matrix through SGD to approximate an optimal lattice. Consider the Cartesian product \(\mathbf{L_{1} \times L_{2}}\), whose generator matrix can be expressed as follows:

\[
G =
\begin{pmatrix}
    G_{1} & 0 \\
    0 & G_{2}
\end{pmatrix},
\]

where \(\mathbf{G_{1}}\) and \(\mathbf{G_{2}}\) are the generator matrices of \(\mathbf{L_{1}}\) and \(\mathbf{L_{2}}\), respectively. When we apply transformation to \(\mathbf{G_{1}}\) and \(\mathbf{G_{2}}\) in the subspace and transformation in big space G. The resulting generator matrix can be written as:

\[
G' =
\begin{pmatrix}
    G_{1} & 0 \\
    0     & G_{2}
\end{pmatrix}
\begin{pmatrix}
    T_{1}U_{1} \\
    T_{2}U_{2}
\end{pmatrix}
,
\]

where $U_1$,$U_2$ is orthogonal matrix, representing the fusion method of two lattice generated matrix $G_1,G_2$ from low dimension.$T_1,T_2$ satisfies $T_iT_i^T=I, i\in{1,2}$. We hope our methods focusing on adjusting the fusion methods. While for the lattice quantizer , the orthogonal transformation will not change its NSM value, cause it's only rotation or reflection in physical. It is easy to see that N dimension matrix $U_{i}$ maintain the property of orthogonal after apply to matrix $T_{i}$:

$$(T_{i}U_{i})(T_{i}U_{i})^T = T_{i}U_{i}U_{i}^TT_{i}^T = T_{i}T_{i}^T = I$$

Therefore, we only need to construct an appropriate form to generate orthogonal matrices. This form should satisfy the following properties:

1. Completeness: The generation form should represent all (or almost all) orthogonal matrices.

2. Optimizability: The parameters of the generation form must be optimizable in a differentiable manner.

In the "Theoretical Preparation" section, we demonstrated that the Householder reflections Matrix Exponential and satisfies Property 1. As for Property 2, empirical observations from our experiments suggest its validity, although we aim to provide a more rigorous theoretical proof in future work.

\subsection{HouseHold Transform}

Actually, the matrix \( \mathbf{U} \) contains \( n^2 \) parameters that need to be learned, which is equivalent to the size of the matrix \( \mathbf{G} \). Based on our experience with gradient descent in similar methods, this high number of parameters often makes convergence challenging. To address this issue, we adopt a reparameterization approach by reformulating the generator matrix as follows:

\[
G'(v_{1}, v_{2}) \leq 
\begin{pmatrix}
    G_{1} & 0 \\
    0     & G_{2}
\end{pmatrix}
\begin{pmatrix}
    T_{1}U_{1}(V1)  \\
    T_{2}U_{2}(V2)
\end{pmatrix},
\]

where \( U_{i}(\mathbf{V_{i}}) \) indicates that \( \mathbf{U} \) is generated from a single vector \( \mathbf{V} \) using Householder transformations. 

Although our theoretical preparation section introduces the construction of symmetry transformations via \( n \) Householder reflections, in our implementation, we simplify the approach by using a single vector \( \mathbf{V} \) as the learnable parameter to generate symmetric orthogonal matrices. While this approach cannot represent all orthogonal matrices, it offers significant advantages. The reduced number of parameters results in faster convergence during training. Despite the potential loss of generality in using simpler symmetric orthogonal matrices, our experiments indicate that they are effective in many scenarios, consistently outperforming direct Cartesian products.

\subsection{Matrix Exponential}

While \( n^2 \) parameters may seem substantial for this task, it is actually manageable for modern machine learning models since \( n \) is typically less than 64. For higher dimensions, the number of sampling points required to evaluate the NSM (Nearest Symmetric Matrix) becomes prohibitively large, and the cost of computing the nearest points is unsustainable. Therefore, we hypothesize that the number of parameters does not significantly affect convergence difficulty in practice.

In our initial experiments with Householder matrices, we observed the effectiveness of orthogonal transformations. Thus, we chose to start with a low-dimensional quantizer and initialize the transformation matrix as orthogonal, which provides a strong starting point for training. To this end, we use the matrix exponential to generate orthogonal matrices for training. However, during the training process, we allow the transformation matrix to deviate from strict orthogonality, enabling it to express greater diversity. Our experimental results further demonstrate the effectiveness of this approach.

\section{Experiments}

\begin{table*}[!ht]
    \centering
    \resizebox{\textwidth}{!}{
        \begin{tabular}{|c|c|c|c|c|c|c|c|c|c|c|}
        \toprule
              ~ & \multicolumn{2}{c}{SotaNSM} &\multicolumn{2}{c}{Generic Bound} & \multicolumn{3}{c}{Best Result of Household Algorithm} & \multicolumn{3}{c|}{Best Result of Matrix Exp Algorithm} \\
        \midrule
            n & NSM & Lattice & Lower & Upper & NSM & Compare & Fusion Lattice & NSM & Compare & Based Lattice \\ 
        \midrule
            12 & $0.07003$ & $Glued D6\otimes D6$ & $0.069179323$ & $0.073098569$ & $0.070510070$ & $<U$ & $E6 \otimes E6$ & $0.070941575$ & $<U$ & $E6 \otimes E6$ \\ \hline
            
            13 & $0.07103$ & $K12\otimes Z$ & $0.068721956$ & $0.072400247$ & $0.06998627$ & $<U<G$ & $K12\otimes Z$ & $0.07077001$ & $<U<G$ & $K12\otimes Z$ \\ \hline
            
            14 & $0.06952$ & $\phi(\epsilon_{7,2}^+)$ & $0.068308096$ & $0.071672217$ & $0.069827377$ & $<U$ & $K12\otimes A_2$ & $0.069999784$ & $<U$ & $K12\otimes A_2^*$ \\ \hline
            
            15 & $0.07037$ & $\phi(\epsilon_{7,2}^+) \otimes Z$ & $0.067931488$ & $0.071008692$ & $0.069461856$ & $<U<G$ & $K12\otimes A_3^*$ & $0.069702447$ & $<U$ & $K12\otimes A_3^*$ \\ \hline
            
            17 & $0.0691$ & $\Lambda 16 \otimes Z$ & $0.067270625$ & $0.069886791$ & $0.06840339$ & $<U<G$ & $\Lambda 16 \otimes Z$ & $0.0682308$ & $<U<G$ & $\Lambda 16 \otimes Z$ \\ \hline
            
            18 & $0.06866$ & $\phi(\epsilon_{9,2}^+)$ & $0.066978741$ & $0.069403282$ & $0.068089$ & $<U<G$ & $\Lambda 16 \otimes A_2$ & $0.068$ & $<U<G$ & $\Lambda 16 \otimes A_2$ \\ \hline
            
            19 & $0.06936$ & $\phi(\epsilon_{9,2}^+) \otimes Z$ & $0.066708503$ & $0.068958664$ & $0.0686972$ & $<U<G$ & $\Lambda 16 \otimes A_3^*$ & $0.06784419$ & $<U<G$ & $\Lambda 16 \otimes A_3^*$ \\ \hline
            
            20 & $0.06769$ & $(32,31)$ & $0.066457468$ & $0.06854849$ & $0.0682606$ & $<U$ & $\Lambda 16 \otimes D_4$ & $0.067881$ & $<U$ & $\Lambda 16 \otimes D_4$ \\ \hline
            
            21 & $0.06836$ & $(32,31)\otimes Z$ & $0.066457468$ & $0.06854849$ & $0.0680651$ & $<U<G$ & $\Lambda 16 \otimes D_5^*$ & $0.067770876$ & $<U<G$ & $\Lambda 16 \otimes D_5^*$ \\ \hline
            
            22 & $0.06853$ & $\phi(\epsilon_{11,2}^+)$ & $0.066004976$ & $0.067826205$ & $0.06849258$ & $<G$ & $\Lambda 16 \otimes E_6^*$ & $0.067177728$ & $<U<G$ & $\Lambda 16 \otimes E_6^*$ \\ \hline
        \end{tabular}
    }
    \caption{The main results of two methods from dimension 12-22,sota results mainly come from \cite{best}\cite{lyu2022betterlatticequantizersconstructed}}
    \label{table1}
\end{table*}

\subsection{Training Experiments}

We completed the main training process for dimensions ranging from 12 to 22 (see results in \ref{table1}). The training results reveal that household reflections typically perform better in lower dimensions, aligning with the simpler characteristics observed in the 12–15 dimension range. On the other hand, the matrix exponential method, with its greater number of learnable parameters, excels in capturing complex combinations and demonstrates superior performance in higher dimensions, particularly in dimensions 21 and 22. Additionally, the experimental results are significantly influenced by the choice of lattice.

The primary challenge we encountered was the evaluation of the Nearest Symmetric Matrix (NSM), which is critical for setting appropriate targets in machine learning. During training, we employed Monte Carlo sampling to evaluate the integration of the NSM.  

To account for the varying number of parameters, we adjusted the number of samples per epoch accordingly. Specifically:  

1. Householder Matrix Training: Since the training complexity for Householder matrices is lower compared to training the entire matrix, we followed prior work by using a single point to compute the NSM.  

2. Matrix Exponential Method: For this approach, each gradient update incorporated hundreds of lattice samples to ensure more stable training progress.  

Further details regarding the training settings are provided in Appendix A. 

\section{Analysis}
The above results show that our method can find lattices with much smaller Normalized Second Moment (NSM), surpassing the previous state-of-the-art results by a huge gap, especially in dimension $17, 18, 19, 21, 22$, where our results are much closer to theoretical lower bounds.

However, our method have the following shortcomings. 

{\bf 1) Unstable training loss}. It is hard to find suitable learning rate to ensure stable training loss reduction. Empirically, we find that schedulers such as \textit{stepLR} in Pytorch helps alleviate this problem.

{\bf 2) Applying orthogonal transformations to sub-lattice components does not guarantee theoretically optimal lattice}. Take the example of the optimal lattice in dimension 3, which is body-centric cubic lattice. This lattice can not be composed by the optimal lattice in dimension 2, which is hexagonal lattice, and the trivial lattice in dimension 1. Since the angle between the basis vector in hexagonal lattice is 60 degrees. This angle remains unchanged through orthogonal transformation to hexagonal lattice. But any 2 basis vectors in body-centric cubic lattice do not form a 60-degree angle.

{\bf 3) Unable to attain exact lattice} The convergence of the algorithm is not qualified so the numerical lattice is hard to converge to an exact lattice which is highly symmetric.It is hard for us to analyze properties of the numerical lattice, e.g. kissing number, the Vonoroi region.Thus, we can only apply Monte Carlo method to calculate the NSM of lattice, which leads to high variance of the result.

Compared to related work, we have several advantages:

{\bf 1) Parameter-efficient} The application of a single Household Transform decreases the complexity of parameters to $O(n)$, which is $O(n^2)$ in \cite{optimization}.This leads to a significant improvement in training efficiency when experimenting on high dimensions.

{\bf 2) Smaller exploration space} The fusion of lattice has a higher NSM and a smaller exploration space than random initialization.The reduction in the number of extrema points in the space makes the method less likely to get stuck in local minima.

\section{Conclusion}

In this paper, we proposed a gradient fusion method for low-dimensional lattice quantizers, leveraging orthogonal transformations to enhance performance. By using Householder reflection matrices and matrix exponentials, we achieved efficient training with reduced parameters, faster convergence, and robust results. Our experiments demonstrated the effectiveness of maintaining orthogonality during training and highlighted the benefits of structured transformations for low-dimensional quantizers. With comprehensive evaluations across various lattices and dimensions, our approach provides a scalable framework for lattice quantization. Future work will focus on extending these methods to higher dimensions and refining theoretical insights to enhance their versatility.

\section*{Accessibility}

All the code is provided. You can visit our project code at this GitHub link: https://github.com/catnanami/lattice-quantizer

\newpage
\appendix
\onecolumn

\section{Training Setting}

For the optimal results of the Household Algorithm, we first initialize the training program with a matrix formed by orthogonal splicing of the currently optimal lattice with two subspaces, as outlined in paper \cite{best}, which specifies the optimal length ratio. We then employ the gradient descent method to train for $10$ epochs, using a learning rate of $5 \times 10^{-3}$. Upon completion of the training, we fix the lattice quantizer and evaluate the normalized second moment (NSM) of the quantizer.

Given a point $x_i$, we utilize integer programming to compute the distance $f(x_i)$ between this point and the nearest integer lattice point. This computation is facilitated through the GurobiPy library in Python.

During training, we begin by randomly selecting a $k$-dimensional point $x_i$ and subsequently calculate $f(x_i)^2$. To ensure that the training is not influenced by scale, we define the loss function as

\[
Loss = f(x_i)^2 \cdot v^{-\frac{2}{k}}.
\]

The parameters of the matrix are then updated via gradient descent. Each epoch consists of $200$ data points.

In the testing phase, we also sample to calculate the NSM. We independently and randomly select $n$ points and compute the average value of $f(x_i)^2$ for these $n$ points, serving as an unbiased estimate of the NSM. To achieve a reliable confidence interval, especially when $n$ is large, we apply the central limit theorem. We recognize that the random variable $X_n = \frac{\sum_{i=1}^{n} f(x_i)^2}{n}$ follows a normal distribution. We estimate the variance of $X_n$ to derive the confidence interval, where the variance $D(x_n) \propto \frac{1}{n}$.For the estimation of $D(X_n)$, we also use the unbiased estimation as following:

\[D(x_n) = 
\frac{1}{n(n-1)} \left(\sum (f(x_i)^2) - \left(\sum \frac{f(x_i)^2}{n}\right)^2\right).
\]

For dimension ranging from $13$ to $21$, we tested $6 \times 10^4$ groups of data and obtained $D(X_n) \leq 2.5 \times 10^{-9}$, resulting in a standard deviation $\leq 5 \times 10^{-5}$.

Consequently, to establish a confidence interval with a width of less than $2 \times 10^{-4}$ , indicating that with a probability of $97.5\%$, the true NSM $\leq X_n + 10^{-4}$, it is necessary to ensure that the number of tests exceeds $6 \times 10^4$. We conducted $10^5$ trials in dimensions $13$, $14$, and $15$, and conducted $(6 \times 10^4)$ trials in dimensions $17$ to $22$, to ensure result accuracy.

For the optimal results of the Matrix Exp Algorithm, leveraging abundant server resources, we performed more than $5 \times 10^5$ trials in each dimension, ensuring that the width of the obtained confidence interval does not exceed $1 \times 10^{-4}$.

The following table provides a detailed overview of the lattices used in our experiments, including the number of updates employed to optimize the models and the confidence bounds for the Nearest Symmetric Matrix (NSM) values. These details are crucial for understanding the robustness and effectiveness of our proposed methods.

\begin{table*}[!ht]
    \centering
    \begin{tabular}{|c|c|c|c|c|c|c|}
    \hline
        ~ & \multicolumn{3}{c|}{Training Setting for Household}  & \multicolumn{3}{c|}{Training Setting for matrix exp} \\ \hline
        n & Iteration & Learning Rate & Confidence Bound & Iteration & Learning Rate & Confidence Bound \\ \hline
        12 & 2000 & 1.00E-03 & 1.00E-05 &  200 & 1.00E-03 & 2.50E-05 \\ \hline
        13 & 2000 & 5.00E-03 & 1.00E-05 &  300 & 1.00E-03 & 2.50E-05 \\ \hline
        14 & 2000 & 5.00E-03 & 1.00E-05 &  300 & 1.00E-03 & 2.50E-05 \\ \hline
        15 & 2000 & 5.00E-03 & 1.00E-05 &  200 & 1.00E-03 & 2.50E-05 \\ \hline
        17 & 2000 & 5.00E-03 & 1.00E-04 &  300 & 1.00E-03 & 2.50E-05 \\ \hline
        18 & 2000 & 5.00E-03 & 1.00E-04 &  460 & 1.00E-03 & 2.50E-05 \\ \hline
        19 & 2000 & 5.00E-03 & 1.00E-04 &  200 & 1.00E-03 & 1.00E-04 \\ \hline
        20 & 2000 & 5.00E-03 & 1.00E-04 &  245 & 1.00E-03 & 5.00E-05 \\ \hline
        21 & 2000 & 5.00E-03 & 1.00E-04 &  240 & 1.00E-03 & 1.00E-04 \\ \hline
        22 & 2000 & 5.00E-03 & 7.00E-05 &  340 & 1.00E-03 & 2.50E-05 \\ \hline
    \end{tabular}
\end{table*}

\end{document}